\documentclass{article}
\pdfoutput=1

\usepackage[preprint,nonatbib]{neurips_2022}

\usepackage[utf8]{inputenc} 
\usepackage[T1]{fontenc}    
\usepackage{hyperref}       
\usepackage{url}            
\usepackage{booktabs}       
\usepackage{amsfonts}       
\usepackage{nicefrac}       
\usepackage{microtype}      
\usepackage{xcolor}         
\usepackage{algorithmic}
\usepackage{algorithm}
\usepackage{graphicx}
\usepackage{subfigure}
\usepackage{easyReview}

\usepackage{marvosym}
\usepackage{ifsym}

\usepackage{amsmath}
\usepackage{amssymb}
\usepackage{mathtools}
\usepackage{amsthm}
\usepackage{bm}
\usepackage[capitalize,noabbrev]{cleveref}
\newcommand{\argmin}{\operatornamewithlimits{arg\,min}}
\theoremstyle{plain}
\newtheorem{theorem}{Theorem}[section]
\newtheorem{proposition}[theorem]{Proposition}
\newtheorem{lemma}[theorem]{Lemma}

\theoremstyle{definition}

\theoremstyle{remark}

\title{A Screening Strategy for Structured Optimization Involving Nonconvex $\ell_{q,p}$ Regularization}

\author{%
  Tiange Li\\
  School of Information Science and Technology\\
  ShanghaiTech University\\
  Shanghai, China\\
  \texttt{litg@shanghaitech.edu.cn} \\
   \AND
   Xiangyu Yang \\
   School of Information Science and Technology\\
   ShanghaiTech University\\
   Shanghai, China\\
   \texttt{yangxy3@shanghaitech.edu.cn} \\
   \And
   Hao Wang\Letter\\
   School of Information Science and Technology\\
   ShanghaiTech University\\
   Shanghai, China\\
   \texttt{haw309@gmail.com} \\
   }

\begin{document}

\maketitle

\begin{abstract}
  In this paper, we develop a simple yet effective screening rule strategy to improve the computational efficiency in solving structured optimization involving nonconvex $\ell_{q,p}$ regularization. Based on an iteratively reweighted $\ell_1$ (IRL1) framework, the proposed screening rule works like a preprocessing module that potentially removes the inactive groups before starting the subproblem solver, thereby reducing the computational time in total. This is mainly achieved by heuristically exploiting the dual subproblem information during each iteration.
  Moreover, we prove that our screening rule can remove all inactive variables in a finite number of iterations of the IRL1 method. Numerical experiments illustrate the efficiency of our screening rule strategy compared with several state-of-the-art algorithms.
\end{abstract}

\section{Introduction}

In modern statistics and statistical machine  \cite{liu2007sparse,yin2015minimization,li2017feature,hastie2019statistical}, many researchers are of practical interest to seek the solutions of the optimization problem involving the empirical risk function with an appropriate penalty term, that is,
\begin{equation}\label{eq: genral_ml_formulation}
	\bm{x}^{*} \in \argmin_{x\in\mathbb{R}^{n}}\ f(\bm{x})+\lambda\mathcal{R}(\bm{x}), 
\end{equation}
where $f:\mathbb{R}^{n}\to\mathbb{R}\cup\{+\infty\}$ is closed and convex and usually refers to the error function or a data fidelity function, while the regularization function $\mathcal{R}:\mathbb{R}^{n}\to\mathbb{R}\cup\{+\infty\}$ is closed and (possibly) nonconvex and nonsmooth. In addition, $\lambda > 0$ is a weighting parameter controlling the trade-off between two terms. For this canonical setting, various combinations of the error function and the regularization function can be obtained to lead to different optimization models, as exemplified by the least squares (with a square loss function) and logistic
regression (with a logistic loss function), etc. In particular, the primary aim of the use of the regularization function $\mathcal{R}$ is to promote desired certain sparse or low-rank structures of the achieved solutions, thereby improving the generalization performance in high-dimensional learning settings. For example, in a highly-cited paper \cite{tibshirani1996regression}, the author proposed the well-known least absolute shrinkage and selection operator (LASSO) method that adopted the $\ell_1$-norm regularizer to remove the irrelevant features of the concerned model, resulting in a simpler and easier model to explain. 
 
\par Since LASSO, the study of the sparsity-promoting regularizers has been a central subject both in statistics and sparse optimization communities owing to their widespread adoption in broad applications \cite{fazel2003log,candes2008enhancing,zhang2010nearly}. During the past decade, nonconvex sparsity-promoting regularizers have attracted considerable attention and demonstrated superior performance than convex counterparts such as the $\ell_1$-norm regularizer. Nonconvex regularizers commonly adopted include smoothly clipped absolute deviation (SCAD) \cite{fan2001variable}, the log-sum penalty \cite{candes2008enhancing},  the Minimax Concave Penalty (MCP) (\cite{zhang2010nearly}) and the $\ell_p$ norm penalty with $0<p<1$ \cite{figueiredo2005bound,figueiredo2007majorization}. 
Meanwhile, as a natural extension, a line of works focused on the structured optimization problem with sparsity-promoting regularizers, which aims to induce the sparsity at the group level, i.e., select or remove simultaneously all the interesting variables forming a group structure. Such group sparsity-promoting regularizers include the variants of the nonsmooth $\ell_{p,q}$-norm by picking different choices of $p>0$ and $q>0$, such as convex $\ell_{2,1}$-norm and $\ell_{\infty,1}$-norm \cite{yuan2006model,bach2011optimization}, and nonconvex $\ell_{2,0.5}$-norm \cite{hu2017group}. 

\par  The screening rule is a technique that is typically used to
alleviate the computation burden in solving problems with
sparsity-inducing regularizers. For the LASSO-type problems, several works considered the safe screening rules by exploiting the dual information to identify zero variables at the optimum \cite{ghaoui2010safe,bonnefoy2014dynamic,ndiaye2017gap}. Here ``safe'' means that those zeros identified by the screening rule are rightly in the optimal solution \cite{ghaoui2010safe}. Recently, a popular safe screening rule, called Gap Safe rule in \cite{ndiaye2017gap}, leverages the duality gap while implementing their screening rule inside the solver. In the meanwhile, a heuristic screening rule \cite{tibshirani2012strong} is mainly achieved by reasonably relaxing the safe screening rule proposed in \cite{ghaoui2010safe} and combining a posterior KKT check to ensure the exact solutions. \cite{ndiaye2017gap,ndiaye2020screening} proposed a safe screening strategy without KKT check. However, the optimization procedure in their methods need to be implemented twice. For the nonconvex sparsity-inducing regularizers,  \cite{lee2015strong} proposed a strong rule for nonconvex MCP and SCAD regularization to scale the data matrix before executing the solver. However, nonsmooth $\ell_{p}$ regularization is not considered in their work. Very recently,  \cite{rakotomamonjy2019screening} proposed a screening rule to tackle a class of nonconvex regularizers within the majorization-minimization framework. As a variation of the Gap Safe rule, their screening rule is used within the subproblem solver as it requires the evaluation of the duality gap and generally prefers using coordinate descent algorithms as subproblem solver. 

The nonconvex $\ell_{q,p}$-regularized problem \cite{hu2017group} with $ q \geq 1$, $0<p<1$ is of key interest in this paper. To address such a nonconvex optimization problem, the iteratively reweighted $\ell_1$ algorithm represents a state-of-the-art approach, which accomplishes the optimization goal by solving a sequence of tractable weighted $\ell_{1}$-regularized subproblems \cite{candes2008enhancing,gasso2009recovering}. Such a class of algorithm and their convergence analysis was extensively treated in many pieces of literature, such as \cite{chen2010convergence,wang2021nonconvex}. However, solvers may be inefficient for solving nonconvex problem with $\ell_{p,q}$ regularization in large-scale optimization. Therefore, in this paper, we are highly motivated to improve the computational efficiency for the iteratively reweighted like method in solving the nonconvex $\ell_{q,p}$-regularized problem. To this end, we propose an efficient screening rule strategy to facilitate the computation of the subproblem solver via screening out the zero variables before solving the subproblem.

In this paper, we propose a novel screening rule strategy, which consists of a heuristic screening test module that is used to early identify zero variables before solving the subproblem and a posterior module that guarantees finding the optimal solution of the subproblem. Specifically, the screening test leverages the KKT condition of the weighted $\ell_{1}$ subproblem. The posterior module is a simple KKT check procedure to ensure an exact solution after solving the subproblem in a lower-dimensional setting. Besides, we explicitly explore the propagation conditions of the proposed screening rule in two successive iterations and we show that the proposed screening rule could safely identify and filter the zero variables in a finite number of iterations. As observed in the numerical studies, IRL1 equipped with the proposed screening rule exhibits significant computational gain.

\subsection{Notation and Preliminaries}
Throughout this paper, we limit our discussion in the real $n$-dimensional Euclidean space $\mathbb{R}^n$. For $\bm{x} \in \mathbb{R}^n$, we use $\bm{x}_{\mathcal{S}}$ to denote a subvector of $\bm{x}$ indexed by an index set $\mathcal{S}$. Let $\bm{0}$ be the zero vector with proper size. $\mathbb{N}$ denotes the set of natural numbers, and we use $[n] \subset \mathbb{N}$ to represent an index set $\{1,2,\ldots,n\}$.

For a continuous function $f:\mathbb{R}^{n}\to\mathbb{R}$, the  subdifferential of $f$ at $\bm{x}\in \textrm{dom}f$ is defined by $\partial f(\bm{x}) = \{ \bm{g} \in\mathbb{R}^{n}\mid 
 f(\bm{x}) + \langle\bm{g}, \bar{\bm{x}} - \bm{x}\rangle \leq f(\bar{\bm{x}}), \forall\bar{\bm{x}} \in\text{dom}f\}$, and any element $\bm{v} \in \partial f(\bm{x})$ is called the subgradient of $f$ at $\bm{x}$. In particular, the subdifferential of $\Vert\bm{x}\Vert$ is
 \begin{equation}\label{eq: subdiff_norm}
 	\partial\Vert\bm{x}\Vert = \{\bm{v}\in\mathbb{R}^{n}\mid\langle\bm{v},\bm{x}\rangle = \Vert\bm{x}\Vert,\ \Vert\bm{v}\Vert_{*}\leq 1\},
 \end{equation}
where $\Vert\bm{x}\Vert_{*} = \sup\limits_{\Vert\bm{u}\Vert \leq 1} \langle\bm{x},\bm{u}\rangle$ is the dual norm of $\Vert\cdot\Vert$. Moreover, the dual of the $\ell_{q}$-norm is the $\ell_{q'}$-norm with $q' = q/(q-1)$. 
\begin{lemma}[Fermet's Rule (see \cite{bauschke2011convex} Proposition 26.1)]
For any convex function $f:\mathbb{R}^d\to\mathbb{R}$
\begin{equation}\label{Fermet's rule}
    \bm{x}^{\star}\in\argmin_{\bm{x}\in\mathbb{R}^d} f(\bm{x}) \Longleftrightarrow 0\in\partial f(\bm{x}^{\star}).
\end{equation}
\end{lemma}
Denote the nonzero index set of a vector $\bm{x}\in \mathbb{R}^{n}$ as $\mathcal{I}(\bm{x}) = \{i\in[n]\mid x_i \neq 0\}$, and its complementary set reads $\mathcal{A}(\bm{x}) = \{i\in[n]\mid x_i = 0\}$. Correspondingly, the zero index set of a vector with group structure $\bm{x}_{\mathcal{G}}\in \mathbb{R}^{n}$ is denoted as $\mathcal{A}(\bm{x}_{\mathcal{G}}) = \{i\in[d]\mid \bm{x}_{\mathcal{G}_{i}} = 0\}$. We use a shorthand notion, $\bm{x}\sim\mathcal{N}(\mu,\sigma^2)$, to describe a Gaussian random variable that is distributed Gaussian with mean $\mu$ and variance $\sigma^2$.

\section{Optimization Model and Algorithm}
\subsection{Nonconvex structured Optimization Problem}

In this section, we briefly review the concerned structured optimization model with a nonconvex $\ell_{q,p}$ regularization and describe an IRL1 method for solving it.

The concerned structured optimization problem involving a square loss and the nonconvex $\ell_{q,p}$ regularization with $ q \geq 1, 0<p<1$ can be described as 
\begin{equation}\label{eq: Group_Lasso}\tag{$\mathcal{P}$}
    \min_{\bm{x}}\ \frac{1}{2}\|\bm{Ax}-\bm{y}\|_2^2+\lambda\Vert\bm{x}\Vert_{q,p}^{p},
\end{equation}
where $\bm{A} \in \mathbb{R}^{m\times n}$ ($m < n$ is often assumed) refers to the feature matrix with column-wise feature  $\bm{a}_j \in \mathbb{R}^{m}$, $\forall j\in [n]$,  $ \Vert\bm{x}\Vert_{q,p}^{p} :=  \sum_{i=1}^d\|\bm{x}_{\mathcal{G}_i}\|_q^p$ in which $\bm{x}\in\mathbb{R}^n$ is divided into $d > 0$ non-overlapping groups $[\bm{x}_{\mathcal{G}_1},\ldots,\bm{x}_{\mathcal{G}_d}]$ with $\mathcal{G} = \{\mathcal{G}_i\}_{i=1}^{d}$ being a partition of $[n]$ and $\bm{y}\in\mathbb{R}^{m}$ is a given observation vector. Similarly, $\bm{A}_{\mathcal{G}_{i}} \in \mathbb{R}^{m\times \vert\mathcal{G}_i\vert}$ represents the features corresponding to the $i$th group.  

\subsection{Basic Ideas of IRL1}

In this work, we primarily focus on using an IRL1 method \cite{wang2021nonconvex} for solving~\eqref{eq: Group_Lasso}. In general, the concerned IRL1 method is an instance of the majorization-minimization procedure. Concretely, to overcome the nonsmooth issue of the $\ell_{q,p}$-norm, a perturbation $\bm{\epsilon} > 0$ is added to have a continuously differentiable regularization term during each iteration.
That is, one can approximate $\Vert\bm{x}\Vert_{q,p}^{p}$ with $\Vert\bm{x} + \bm{\epsilon}\Vert_{q,p}^{p}$. 
Then, at the $k$th iterate, it holds that
\begin{equation}\label{eq: property_concavity}
	\begin{aligned}
	  \sum_{i=1}^d\|\bm{x}_{\mathcal{G}_i}+\epsilon_i^{k}\|_q^p
	  \leq \sum_{i=1}^d \|\bm{x}_{\mathcal{G}_i}^{k} + \epsilon_{i}^{k} \|_q^p+p(\|\bm{x}_{\mathcal{G}_i}^{k}\|_q+\epsilon_i^k)^{p-1}(\|\bm{x}_{\mathcal{G}_i}\|_q-\|\bm{x}_{\mathcal{G}_i}^{k}\|_q),
	\end{aligned}
\end{equation}
where the inequality is true by the concavity of $(\cdot)^{p}$. Therefore, the next iterate $\bm{x}^{k+1}$ is obtained as 
\begin{equation}\label{eq: weighted_l1_sub}
	\begin{aligned}
		\bm{x}^{k+1} \!=\! \argmin_{\bm{x}\in\mathbb{R}^{n}} \left\lbrace  \frac{1}{2}\|\bm{Ax}\!-\!\bm{y}\|_2^2\!+\!
		 \lambda \sum_{i=1}^d w_{i}^{k}\|\bm{x}_{\mathcal{G}_i}\|_q\right\rbrace
	\end{aligned}
\end{equation}
with $w_{i}^{k} = p(\|\bm{x}_{\mathcal{G}_i}^{k}\|_q+\epsilon_{i}^{k})^{p-1}$. Typically, the perturbation $\bm{\epsilon}$ is supposed to deacy to $\bm{0}$ as the algorithm proceeds to ensure the global convergence. For completeness, we summarize the IRL1 method in~\Cref{eq: Algo1}.
\begin{algorithm}[htbp]
\caption{An Iteratively Reweighted $\ell_{1}$ Algorithm for Solving~\eqref{eq: Group_Lasso}}
\label{eq: Algo1}
\begin{algorithmic}[1]
\REQUIRE $\mu \in (0,1)$, $\bm{\epsilon}^0\in\mathbb{R}^{d}_{++}$ and $\bm{x}^0\in\mathbb{R}^m$.
\STATE Set $k=0$.
\REPEAT
	\STATE	Compute~$w_{i}^{k} = p(\|\bm{x}_{\mathcal{G}_i}^{k}\|_q+\epsilon_{i}^{k})^{p-1}$, $\forall i \in [d]$.
    \STATE  Solve~\eqref{eq: weighted_l1_sub} for $\bm{x}^{k+1}$.
   \STATE Set $\bm{\epsilon}^{k+1}\leq\mu\bm{\epsilon}^{k}$ and set $k\leftarrow k+1$.
\UNTIL{convergence}
\end{algorithmic}
\end{algorithm}


\section{Proposed Screening Rules}

In this section, we develop a screening rule that aims to filter the null group features in the subproblem solution in advance, enabling the subproblem to be solved in a reduced space and thus accelerating the entire computation. The design of such a rule are inspired by a heuristic use of the dual information of the subproblem. 
\subsection{A Priori Screening Test Procedure}

The proposed heuristic screening rule is motivated by exploiting the dual information of the subproblem associated with~\eqref{eq: weighted_l1_sub}. Specifically, by dropping the superscript $k$, the $k$th primal subproblem~\eqref{eq: weighted_l1_sub} can be equivalently transformed into
\begin{equation}\label{eq: primal_sub}
\min_{\bm{x} \in \mathbb{R}^n}\ P(\bm{x}):= \frac{1}{2}\|\bm{Ax}-\bm{y}\|_{2}^2+\sum_{i=1}^d \lambda_{i}\|\bm{x}_{\mathcal{G}_i}\|_q.
\end{equation}
where $\lambda_{i} = \lambda w_{i} > 0$ represents the group-wise weighting parameter for any $i \in [d]$.
By letting $\bm{z} = \bm{Ax}-\bm{y}$, we can rewrite problem~\eqref{eq: primal_sub} into
\begin{equation}\label{eq: Reformulation}
	\begin{aligned}
		\min_{\bm{x} \in \mathbb{R}^n,\bm{z} \in \mathbb{R}^m} &\ \frac{1}{2}\|\bm{z}\|_{2}^2+\langle \lambda_{\mathcal{G}}^T,\|\bm{x}_{\mathcal{G}}\|_q\rangle\\
		\textrm{s.t.}\quad\ \ &\  \bm{z} = \bm{Ax} - \bm{y},
	\end{aligned}
\end{equation}
where $\|\bm{x}_{\mathcal{G}}\|_q = [\|\bm{x}_{\mathcal{G}_1}\|_q,\ldots\|\bm{x}_{\mathcal{G}_d}\|_q]^{T}$ and $\lambda_\mathcal{G} = [\lambda_{\mathcal{G}_1},\ldots,\lambda_{\mathcal{G}_d}]^{T}$.  Hence, the Lagrangian associated with~\eqref{eq: Reformulation} reads
\begin{equation*}
	\begin{aligned}
		\mathcal{L}_{\lambda_{\mathcal{G}}}(\bm{z},\bm{x};\bm{\theta}) &= \frac{1}{2}\bm{z}^T\bm{z}+ \langle \lambda_{\mathcal{G}}^T,\|\bm{x}_{\mathcal{G}}\|_q\rangle + \bm{\theta}^T(\bm{Ax}-\bm{y}-\bm{z}) = \mathcal{L}_1(\bm{z})+\mathcal{L}_2(\bm{x}),
	\end{aligned}
\end{equation*}
where $\bm{\theta} \in \mathbb{R}^{m}$ is the Lagrange multiplier associated with~\eqref{eq: Reformulation}, and
\begin{equation*}
	\begin{aligned}
		\mathcal{L}_1(\bm{z}) &= \frac{1}{2}\bm{z}^T\bm{z} - \bm{\theta}^T\bm{z}-\bm{\theta}^T\bm{y},\quad \mathcal{L}_2(\bm{x}) = \langle \lambda_{\mathcal{G}}^T,\|\bm{x}_{\mathcal{G}}\|_q\rangle + \bm{\theta}^T\bm{Ax}.
	\end{aligned}
\end{equation*}
\par Then, we have the Lagrange dual function $G:\mathbb{R}^{m}\to\mathbb{R}$:
\begin{equation*}
	G(\bm{\theta}) =\inf\mathcal{L}_1(\bm{z}) + \inf \mathcal{L}_2(\bm{x}).
\end{equation*}
Since $\mathcal{L}_1(\bm{z})$ is a convex quadratic function with respect to $\bm{z}$, it is easy to have
\begin{equation}
	\begin{aligned}
		\inf \mathcal{L}_1(\bm{z}) 
		= -\frac{1}{2}&\bm{\theta}^T\bm{\theta} - \bm{\theta}^T\bm{y} = -\frac{1}{2}\|\bm{\theta}+\bm{y}\|_2^2 +\frac{1}{2}\|\bm{y}\|_2^2.
	\end{aligned}
\end{equation}
\par On the other hand, we have
\begin{equation}
	\begin{aligned}
		\mathcal{L}_2(\bm{x}) &= \sum_{i=1}^d \lambda_{\mathcal{G}_i}\|\bm{x}_{\mathcal{G}_i}\|_q + (\bm{\theta}^T\bm{A})_{\mathcal{G}_i}\bm{x}_{\mathcal{G}_i} \geq \sum_{i=1}^d \lambda_{\mathcal{G}_i}\|\bm{x}_{\mathcal{G}_i}\|_q - |(\bm{\theta}^T\bm{A})_{\mathcal{G}_i}\bm{x}_{\mathcal{G}_i}|\\
		&\geq \sum_{i=1}^d \lambda_{\mathcal{G}_i}\|\bm{x}_{\mathcal{G}_i}\|_q - \|(\bm{\theta}^T\bm{A})_{\mathcal{G}_i}\|_{q'}\|\bm{x}_{\mathcal{G}_i}\|_q=\sum_{i=1}^d \|\bm{x}_{\mathcal{G}_i}\|_q(\lambda_{\mathcal{G}_i} - \Vert(\bm{\theta}^T\bm{A})_{\mathcal{G}j}\Vert_{q'}),
	\end{aligned}
\end{equation}
where the second inequality holds by H\"{o}lder's inequality and $q'$ satisfies $q' = q/(q-1)$. Note that we concentrate on the case in which $\lambda_{\mathcal{G}_i} \geq \|(\bm{\theta}^T\bm{A})_{\mathcal{G}_i}\|_{q'} $, otherwise $\inf \mathcal{L}_2(\bm{x})$ attains $-\infty$. Correspondingly, the Lagrange dual problem of~\eqref{eq: primal_sub} reads
\begin{equation}\label{eq: dual_sub}
	\begin{aligned}
		\max_{\bm{\theta}\in\mathbb{R}^{m}} &\quad G(\bm{\theta}) =  -\frac{1}{2}\|\bm{\theta}+\bm{y}\|_2^2 +\frac{1}{2}\|\bm{y}\|_2^2 \\
		\text{s.t.}  \ &\quad
		\lambda_{i} \geq \|(\bm{\theta}^T\bm{A})_{\mathcal{G}_i}\|_{q'}, \quad\forall i\in [d].
	\end{aligned}
\end{equation}

\par We introduce the proposed screening rule by first considering an extreme case that the optimal solution $\bm{x}^{*}$ of~\eqref{eq: primal_sub} is $\bm{0}$, and denote $\bm{\lambda}^0$ as the tuning parameter. By the strong duality, at $(\bm{x}^{*},\bm{\theta}^{*})$, we have $\bm{\theta}^{*} = -\bm{y}$, and this is because the objectives of the primal subproblem~\eqref{eq: primal_sub} is equal to the dual subproblem~\eqref{eq: dual_sub}. Meanwhile, by the dual feasibility of $\bm{\theta}^{*}$, we have
\begin{equation}\label{eq: ZerosolConditon}
	\|((\bm{\theta}^{*})^T\bm{A})_{\mathcal{G}_i}\|_{q'} = \|(\bm{y}^T\bm{A})_{\mathcal{G}_i}\|_{q'} \leq {\lambda}_i^0,\ \forall i \in [d].
\end{equation}

Thereafter, one can obtain the following lemma :
\begin{lemma}
For problem~\eqref{eq: primal_sub}, it holds that
\begin{equation}
    0 \in \argmin_{\bm{x}\in\mathbb{R}^{n}}  P(\bm{x})\Longleftrightarrow \lambda_i\geq\lambda_i^0,\ \forall i \in [d].
\end{equation}
\end{lemma}
\begin{proof}
The proof of sufficiency lies in~\eqref{eq: ZerosolConditon}.
Then, suppose for all $i \in[d]$, there exists $\bm{\beta}_i\in\mathbb{R}^{d}$ such that 
$$0=-(\bm{y}^T\bm{A})_{\mathcal{G}_i}+\lambda_i\bm{\beta}_i.$$
Then, one can obtain 

$$\|\bm{\beta}_i\|_{q'}=\frac{1}{\lambda_i}\|(\bm{y}^T\bm{A})_{\mathcal{G}_i}\|_{q'}\leq 1,\forall i \in[d],$$

where the inequality holds true by $\lambda_i\geq\|(\bm{y}^T\bm{A})_{\mathcal{G}_i}\|_{q'}$.
Equation above shows that $\bm{\beta}_i\in\partial\|0\|_q$, thus it is easy to have that

$$\bm{0}\in -\bm{A}_{\mathcal{G}_i}\bm{y}+ \lambda_{i}\partial\|0\|_q,\forall i \in[d].$$

By Fermet's Rule~\eqref{Fermet's rule}, we have
$0\in\mathop{\arg\min\limits_{\bm{x}\in\mathbb{R}^{n}}}P(x)$. This completes the proof.


\end{proof}
Thus, $\lambda_i<\lambda_i^0,\forall i\in[d]$ can be the only case that we can focus on. By the way, denote $\tilde{\lambda}_{i}^{\max}=\lambda_i^0,\forall i \in [d]$ as the smallest $\bm{\lambda}$ with which $\bm{0}$ is the optimal solution.
That being said, we have $\bm{x}_{\mathcal{G}_{i}}^{*} = \bm{0}$ in \eqref{eq: primal_sub} if $\lambda_{i} \geq \tilde{\lambda}_{i}^{\max}$ holds, $\forall i \in [d]$. However, such an ideal screening rule \eqref{eq: ZerosolConditon} is generally impractical since such conditions are difficult to be satisfied. 

To bypass this issue, in our work, we propose to use a practical screening condition for filtering null groups, i.e.,
\begin{equation}\label{eq: rule_final}
     \|(\bm{y}^T\bm{A})_{\mathcal{G}_i}\|_{q'} \leq \lambda_{i}, \forall i\in [d].
\end{equation}


\par In a pioneer work \cite{tibshirani2012strong}, the authors proposed a strong rule for convex LASSO-type problems. For \eqref{eq: weighted_l1_sub}, strong rule suggests discarding variables $\bm{x}_{\mathcal{G}_i}$ if $\|(\bm{y}^T\bm{A})_{\mathcal{G}_i}\|_{q'} < w_i(2\lambda-\lambda_{\max})$ with $\lambda$ being the tuning parameter in~\eqref{eq: weighted_l1_sub} and $\lambda_{\max}$ correspondingly the smallest tuning parameter which produces solution $\bm{0}$. That rule can be equally written as
\begin{equation}\label{eq: strong_rule}
    \|(\bm{y}^T\bm{A})_{\mathcal{G}_i}\|_{q'} < 2\lambda_{i}-w_{i}\max_i\{\frac{\|(\bm{y}^T\bm{A})_{\mathcal{G}_i}\|_{q'}}{w_{i}}\}.
\end{equation}
It is obvious that the proposed screening rule~\eqref{eq: rule_final} is much simpler than~\eqref{eq: strong_rule}, since rule~\eqref{eq: rule_final} does not involve the computation of $\max_i\{\frac{\|(\bm{y}^T\bm{A})_{\mathcal{G}_i}\|_{q'}}{w_{i}}\}$. On the other hand, strong rule explicitly requires the regularization parameter $\lambda>\frac{1}{2}\max_i\{\frac{\|(\bm{y}^T\bm{A})_{\mathcal{G}_i}\|_{q'}}{w_{\mathcal{G}_i}}\}$ in \eqref{eq: weighted_l1_sub}, which may generally limit its use for an appropriate regularization parameter in practice. Instead, the proposed rule~\eqref{eq: rule_final} does not impose such a requirement of $\lambda$, enabling possibly screening out more null feature groups.

We should highlight that the proposed screening rule~\eqref{eq: rule_final} is practical yet efficient, which is confirmed in the numerical studies. Moreover, we mention that the proposed heuristic screening rule~\eqref{eq: rule_final} is not completely safe and hence we cannot ensure discard the features correctly. Nevertheless, we can leverage the KKT conditions of~\eqref{eq: primal_sub} to check the wrongly filtered variables, which will be discussed in the next subsection.
\subsection{A Posterior KKT Check Procedure}

To prevent from discarding the null feature groups mistakenly, we follow the similar spirit in \cite{tibshirani2012strong} to combine a KKT check step. For this purpose, we note that the KKT optimality conditions of \eqref{eq: primal_sub} reads
\begin{equation}\label{eq: pri_sub_optimality}
    \bm{0} \in \bm{A}_{\mathcal{G}_i}^{T}(\sum_{i=1}^{d}\bm{A}_{\mathcal{G}_i}\bm{x}_{\mathcal{G}_i} -\bm{y}) + \lambda_{i}\partial\Vert\bm{x}_{\mathcal{G}_i}\Vert_q,\forall i\in [d].
\end{equation}
For detecting wrong screening, suppose we set $\bm{x}_{\mathcal{G}_i}=0$. Then the KKT condition reads  
\begin{equation}
	-\bm{A}_{\mathcal{G}_i}^{T}(\sum_{i=1}^{d}\bm{A}_{\mathcal{G}_i}\bm{x}_{\mathcal{G}_i} -\bm{y}) = \lambda_{i}\bm{\beta},
\end{equation}
where $\bm{\beta}$ is supposed to be in $\partial\|0\|_{q}$. Thus, we can use whether $\bm{\beta}\in\partial\|0\|_{q}$ to check if KKT condition is satisfied. The corresponding KKT check procedure is 
\begin{equation*}
    \|\frac{-\bm{A}_{\mathcal{G}_i}^{T}(\sum_{i=1}^{d}\bm{A}_{\mathcal{G}_i}\bm{x}_{\mathcal{G}_i} -\bm{y})}{\lambda_{i}}\|_{q'}\leq 1, \forall i \in \mathcal{A}(\bm{x}),
\end{equation*}
which is equivalent to 
\begin{equation}\label{eq: KKT_check_}
    \|\bm{A}_{\mathcal{G}_i}^{T}(\sum_{i=1}^{d}\bm{A}_{\mathcal{G}_i}\bm{x}_{\mathcal{G}_i} -\bm{y})\|_{q'}\leq \lambda_i, \forall i \in \mathcal{A}(\bm{x}).
\end{equation}
Therefore, we can use~\eqref{eq: KKT_check_} to detect the wrongly discarded variables after applying the proposed screening rule~\eqref{eq: rule_final} and having the optimal solution of~\eqref{eq: primal_sub}. Specifically, if~\eqref{eq: KKT_check_} is violated, we add the corresponding filtered group back and continue to solve~\eqref{eq: primal_sub} till condition \eqref{eq: KKT_check_} is satisfied. 
Overall, the proposed heuristic screening rule to improve computational efficiency of the IRL1 algorithm for solving~\eqref{eq: primal_sub} is summarized in~\Cref{algo2}.


\begin{algorithm}[htbp]
\caption{Proposed Screening Rule to Accelerate IRL1 Method}
\label{algo2}
\begin{algorithmic}[1]
	\REQUIRE  $\bm{\lambda}_{\mathcal{G}}$,  $\bm{A}\in\mathbb{R}^{m\times n}$, $\bm{y}\in\mathbb{R}^{m}$, $\textit{list} \subset [d]$ and $\textit{scrlist}\subset [d]$.
	\STATE \textbf{Screening Condition}:
	\IF{\eqref{eq: rule_final} holds}
	\STATE $\textit{list} \leftarrow \textit{list} \backslash \{j\}$ and $\textit{scrlist} \leftarrow \textit{scrlist} \cup  \{j\}$.
	\ENDIF
	\STATE \textbf{Subproblem Solution}:
	\STATE Solve~\eqref{eq: primal_sub} with $\bm{A}_{\textit{list}}$ to obtain $\hat{\bm{x}}$ and set $\bm{x}_{\textit{list}}\leftarrow \hat{\bm{x}} $.
	\STATE \textbf{KKT Check}:
	\STATE Set $\textit{errlist}\leftarrow \phi$.
	\IF{\eqref{eq: KKT_check_} is not satisfied}
	\STATE $\textit{scrlist} \leftarrow \textit{scrlist} \backslash \{i\}$ and $\textit{errlist} \leftarrow \textit{errlist} \cup \{i\}$.
	\ENDIF
	\IF{\textit{errlist} is not $\phi$}
	\STATE $\textit{list} \leftarrow \textit{errlist} \cup \textit{list}$ and go to step 5.
	\ENDIF
	\STATE \textbf{Output}:$\bm{x}$, \textit{list} and \textit{scrlist}.
\end{algorithmic}
\end{algorithm}
We briefly discuss the differences and connections between the proposed screening rules and other existing rules to close this section. Unlike the most existing dynamic screening rules \cite{bonnefoy2014dynamic,fercoq2015mind,ndiaye2017gap,rakotomamonjy2019screening,ndiaye2020screening}, the proposed screening rule works before starting the solver for solving the convex subproblem \eqref{eq: primal_sub} during each iteration. This feature allows us to combine the proposed screening rule with those dynamic screening rules \cite{ndiaye2017gap,ndiaye2020screening} to reduce the dimension of input data or generate an approximate solution to warm start the convex problem solver equipped with those dynamic rules. Another line works considered the static screening rules. In particular, we show that our proposed screening rule is superior than that proposed in \cite{tibshirani2012strong}, and the screening rule proposed in \cite{lee2015strong} cannot be applied to nonconvex $\ell_p$ regularization problems due to its nonsmooth nature. 

\section{IRL1 Algorithm with Screening Rule Strategy}
In this section, we provide the theoretical analysis for the proposed screening rule that is utilized within the IRL1 algorithmic framework. We first show that the current filtered groups can be fully detected in the next iteration. Then, we prove that the proposed rule can detect all null groups in a finite number of iterations. 
\subsection{Screening Between Iterations}
\par The following lemma states that once the variables enter the screened list in the current iteration, then it is guaranteed to be detected by the proposed screening rule in the next iteration.
\begin{lemma}
Consider the $k$th subproblem~\eqref{eq: primal_sub} with a fixed weighting parameter $\lambda_{i}^{k}$ for $i \in [d]$. Let $\{\mathcal{G}_{\mathcal{S}}\}_{\mathcal{S}\subset [d]}$ be the screened list returned by~\Cref{algo2} that is implemented at the $k$th subproblem. Then,  it holds that
\begin{equation}
	\|(\bm{y}^T\bm{A})_{\mathcal{G}_{i}}\|_{q'} < \lambda_{i}^{k+1}, \ \forall i \in \mathcal{S},
\end{equation}
\end{lemma}
\begin{proof}
	At the $k$th subproblem, for any $i \in \mathcal{S}$ , we know $\bm{x}_{ \mathcal{G}_i}^{k+1} = \bm{0}$. Therefore, we have 
	\begin{equation}
		\|\bm{x}_{\mathcal{G}_i}^{k+1}\|_q +\epsilon_i^{k+1} = \epsilon_i^{k+1} < \|\bm{x}_{\mathcal{G}_i}^{k}\|_q +\epsilon_i^{k},
	\end{equation}
where the inequality holds since $\bm{\epsilon}$ decreases in each iteration. As a result, it holds that
\begin{equation*}
	 \lambda_{i}^{k+1} = \lambda p\|\bm{x}_{\mathcal{G}_i}^{k+1}+\epsilon_i^{k+1}\|_q^{p-1}>\lambda p\|\bm{x}_{\mathcal{G}_i}^{k}+\epsilon_i^{k}\|_q^{p-1} = \lambda_{i}^{k},
\end{equation*}
where the inequality holds since $(\cdot)^{p-1}$ monotonically decreases over $\mathbb{R}_{++}$.  Therefore,
\begin{equation}
	\|(\bm{y}^T\bm{A})_{\mathcal{G}_{i}}\|_{q'} \leq \lambda_{i}^{k} < \lambda_{i}^{k+1},
\end{equation}
as desired. This completes the proof.
\end{proof}

Overall, we now state the IRL1 algorithm equipped with our proposed screening rue in~\Cref{algo: overall}.

\begin{algorithm}[H]
\caption{IRL1 with Proposed Screening Rule}
\label{algo: overall}
\begin{algorithmic}[1]
\REQUIRE $\mu \in (0,1)$, $\lambda >0 $,  $\bm{x}^0\in\mathbb{R}^n_{++}$ ,  $\bm{\epsilon}^0\in\mathbb{R}^d_{++}$,  $\bm{A}\in\mathbb{R}^{m\times n}$ and $\bm{y}\in\mathbb{R}^{m}$.
\STATE Set $k=0$, $\textit{list}\leftarrow[d]$, $\textit{scrlist}\leftarrow\phi$.
\REPEAT
    \STATE   Compute~$w_{\mathcal{G}_i}^{k} = p(\|\bm{x}_{\mathcal{G}_i}^{k}\|_q+\epsilon_{i}^{k})^{p-1}$, $\forall i \in [d]$.
   \STATE Set $\lambda_{\mathcal{G}_i}^k \leftarrow \lambda w_{\mathcal{G}_i}^{k} $, $\forall i \in [d]$.
   \STATE Call~\Cref{algo2} with $\bm{A}_{list}$ to obtain $\bm{x}^{k+1}$ and update $\textit{list}$, $\textit{scrlist}$.
   \STATE Set $\bm{\epsilon}^{k+1}\leq\mu\bm{\epsilon}^{k}$ and set $k\leftarrow k+1$.
\UNTIL{convergence}
\end{algorithmic}
\end{algorithm}

In the step $5$, we use the proximal gradient method proposed in \cite{wright2009sparse} for solving~\eqref{eq: primal_sub}, which admits an efficient soft-thresholding operation. Moreover, we follow \cite{wright2009sparse} to use the warm-start technique to facilitate the subproblem solution, which leverages the last iterate to initialize the next subproblem.


\subsection{Algorithm Analysis}
The global convergence results of the basic IRL1 algorithm for generalized nonconvex problem with convex constraints is shown in \cite{wang2021nonconvex}. The convergence theory also applies to problem~\eqref{eq: Group_Lasso}. Note that the proposed screening strategy accelerates solving each subproblem by reducing the data dimension while keep the solution still optimal to each subproblem. In consequence the convergence of the IRL1 algorithm is not influenced by the proposed strategy. In this section, we show that the proposed screening rule can discard all null group features in a finite number of iterations. The result indicates that the faster $\epsilon$ decreases, the faster the proposed screening rule detects all null groups.
\begin{lemma}\label{C}
Let $\{\bm{x}^k\}$ be the sequence generated by~\Cref{algo: overall}. Then, there exists $C>0$ such that 
$$\|\bm{A}_{\mathcal{G}_i}^T(\sum_{j=1}^{d}\bm{A}_{\mathcal{G}_j}\bm{x}_{\mathcal{G}_j}^{k} -\bm{y})\|_{q'} < C,\forall k\in\mathbb{N},\forall i \in[d].$$
\end{lemma}
This lemma relies on a direct consequence of the global convergence of~\Cref{algo: overall}.
\begin{lemma}\label{Czero}
Let $\{\bm{x}^k\}$ be the sequence generated by~\Cref{algo: overall} and constant $C$ is defined in~\cref{C}. Then it holds true that if $\lambda_i^{\tilde{k}} \geq C$ for some $\tilde{k}\in\mathbb{N}$, then $\bm{x}^k_{\mathcal{G}_i} = 0$ for all $k\geq\tilde{k}$.
\begin{proof}
The KKT condition reads
\begin{equation}
    -\bm{A}_{\mathcal{G}_i}^T(\sum_{j=1}^{d}\bm{A}_{\mathcal{G}_j}\bm{x}_{\mathcal{G}_j}^{k} -\bm{y})=\lambda_i^{\tilde{k}}\bm{\beta},
\end{equation}
where $\bm{\beta} \in \partial\|\bm{x}_{\mathcal{G}_i}\|_q$.
Then, it equivalently holds true that 
\begin{equation}
    \frac{1}{\lambda_i^{\tilde{k}}}\|\bm{A}_{\mathcal{G}_i}^T(\sum_{j=1}^{d}\bm{A}_{\mathcal{G}_j}\bm{x}_{\mathcal{G}_j}^{k} -\bm{y})\|_{q'}= \|\bm{\beta}\|_{q'}.
\end{equation}
If $\bm{x}_{\mathcal{G}_i} \neq 0$, then 
$$1=\|\bm{\beta}\|_{q'}=  \frac{1}{\lambda_i^{\tilde{k}}}\|\bm{A}_{\mathcal{G}_i}^T(\sum_{j=1}^{d}\bm{A}_{\mathcal{G}_j}\bm{x}_{\mathcal{G}_j}^{k} -\bm{y})\|_{q'}.$$ 
That contradicts $\lambda_i^{\tilde{k}} \geq C>\|\bm{A}_{\mathcal{G}_i}^T(\sum_{j=1}^{d}\bm{A}_{\mathcal{G}_j}\bm{x}_{\mathcal{G}_j}^{k} -\bm{y})\|_{q'},\forall k\in\mathbb{N}$. Therefore, $\bm{x}_{\mathcal{G}_i} = 0.$\\
By induction we know that $\bm{x}_{\mathcal{G}_i}^{\tilde{k}}\equiv0$ for any $k > \tilde{k}$. This completes the proof.
\end{proof}
\end{lemma}
\begin{lemma}\label{stable sign}
Let $\{\bm{x}^k\}$ be the sequence generated by~\Cref{algo: overall} and constant $C$ is defined in~\cref{C}. Then, there exist sets $\mathcal{A}^{\star}\subset\{1,2,\ldots,n\}$ and $\bar{k}>0$, such that $\forall k >\bar{k}$, $\mathcal{A}(\bm{x}^k)=\mathcal{A}^{\star}$.
\end{lemma}
\begin{proof}
Suppose by contradiction this statement is not true. There exists $j\in[d]$ such that $\bm{x}_{\mathcal{G}_j}$ takes zero and nonzero value both for infinite times. Hence, there exists a subsequence $\mathcal{S}_1\cup\mathcal{S}_2=\mathcal{N}$ such that $|\mathcal{S}_1|=\infty,|\mathcal{S}_2|=\infty$, and that
\begin{equation}
    \bm{x}_{\mathcal{G}_j}^k=0,\forall k \in \mathcal{S}_1,\bm{x}_{\mathcal{G}_j}^k\neq0,\forall k \in \mathcal{S}_2.
\end{equation}
Since $\epsilon$ is monotonically decreasing to 0, there exists $\tilde{k}>0$ such that
\begin{equation}
    \lambda_j^{\tilde{k}}=\lambda\cdot p(\|\bm{x}_{\mathcal{G}_j}^{\tilde{k}}\|_q+\epsilon)^{p-1}=\lambda\cdot p(\epsilon)^{p-1}>C.
\end{equation}
It follows by~\cref{Czero} that $ \bm{x}_{\mathcal{G}_j}^k\equiv0$ for any $k>\tilde{k}$, which implies $\{\tilde{k}+1,\tilde{k}+2,\ldots\}\subset\mathcal{S}_1$ and $|\mathcal{S}_2|<\infty$. This violates the assumption $|\mathcal{S}_2|=\infty$. Hence~\cref{stable sign} is true.
\end{proof}

\begin{proposition}
Let $\{\bm{x}^k\}$ be the sequence generated by~\Cref{algo: overall}. 
Then, there exists $\bar{k}\in \mathbb{N}$ such that for any $ k \geq \bar{k}, k \in \mathbb{N}$, it holds that for any $i \in  \mathcal{A}(\bm{x}_{\mathcal{G}}^k)$
\begin{equation*}
	 \|(\bm{y}^T\bm{A})_{\mathcal{G}_i}\|_{q'} \leq \lambda_{i}^k\ \textrm{ and }\ \Vert\bm{A}_{\mathcal{G}_i}^{T}(\sum_{j=1}^{d}\bm{A}_{\mathcal{G}_j}\bm{x}_{\mathcal{G}_j}^{k} - \bm{y})\Vert_{q'} \leq \lambda_{i}^{k}.
\end{equation*}
\end{proposition}
\begin{proof}
By~\cref{stable sign}, there exists $k_1\in\mathbb{N}$ such that $\forall k \geq k_1,\mathcal{A}(\bm{x}^{k})$ remains stable. Then, since $\bm{\epsilon}$ decreases in each iteration,
$\forall i \in \mathcal{A}(\bm{x}_{\mathcal{G}}^{*})$, there exists $k_2\in\mathbb{N}$ such that $$\forall k \geq k_2,\quad \epsilon^{k}_i \leq (\frac{\max(C_1,C)}{\lambda p})^{1/(p-1)}$$ with $C_1 = \max\limits_{i\in[d]}\|(\bm{y}^TA)_{\mathcal{G}_i}\|_{q'}$ and
$C$ is defined in~\cref{C}. Therefore, we have 
\begin{equation}\label{eq:safe}
	\begin{aligned}
		\lambda_{\mathcal{G}_i}^{k} = \lambda p(\epsilon^{k}_i)^{p-1} &\geq \lambda p((\frac{\max(C_1,C)}{\lambda p})^{\frac{1}{p-1}})^{p-1} = \max(C_1,C),
	\end{aligned}
\end{equation}
which indicates that the null group features are safely screened out by~\eqref{eq: rule_final} and~\eqref{eq: KKT_check_}. Since $\lambda_{\mathcal{G}_i}^{k}>C,\forall k >k_2$, by~\cref{Czero}, it can be obtained that $k_2\geq k_1$. Hence, the proposition is true and $\bar{k}=k_2$.
\end{proof}


\section{Numerical Experiments}

In this section, we conduct extensive experiments on synthesised data and real-world data to illustrate the substantial gains in computational efficiency brought by the proposed screening rule strategy. All numerical experiments are implemented in Matlab R2020b and executed on Macbook Air (Intel Core i7, 1.2GHz, 16GB of RAM). On the experiment setup, we initialize $\epsilon_0 =(\frac{\lambda_{\max}}{2p\lambda})^{\frac{1}{p-1}}$ and set $\mu = 0.9$ for the IRL1 algorithm. The starting point $\bm{x}^{0}$ is initialized as the solution of the $\ell_{q,1}$ regularization problem with early stopping. We determine the weighting parameter $\lambda$ using the grid search of $\{\lambda_t\} = \{10^{-(1+\frac{2t}{Q-1})}\lambda_{\max}\}$ with $t \in\{0,1,\ldots, Q-1\}$. Moreover, we terminate \Cref{algo: overall} if
$\frac{\|\bm{x}^{k+1}-\bm{x}^{k}\|}{\|\bm{x}^{k+1}\|} \leq 10^{-6}$ was satisfied.
\begin{figure}[htbp]
\centering
\subfigure[Low-dimension data]{
\begin{minipage}[b]{0.48\textwidth}
\centering 
\includegraphics[width=1\textwidth]{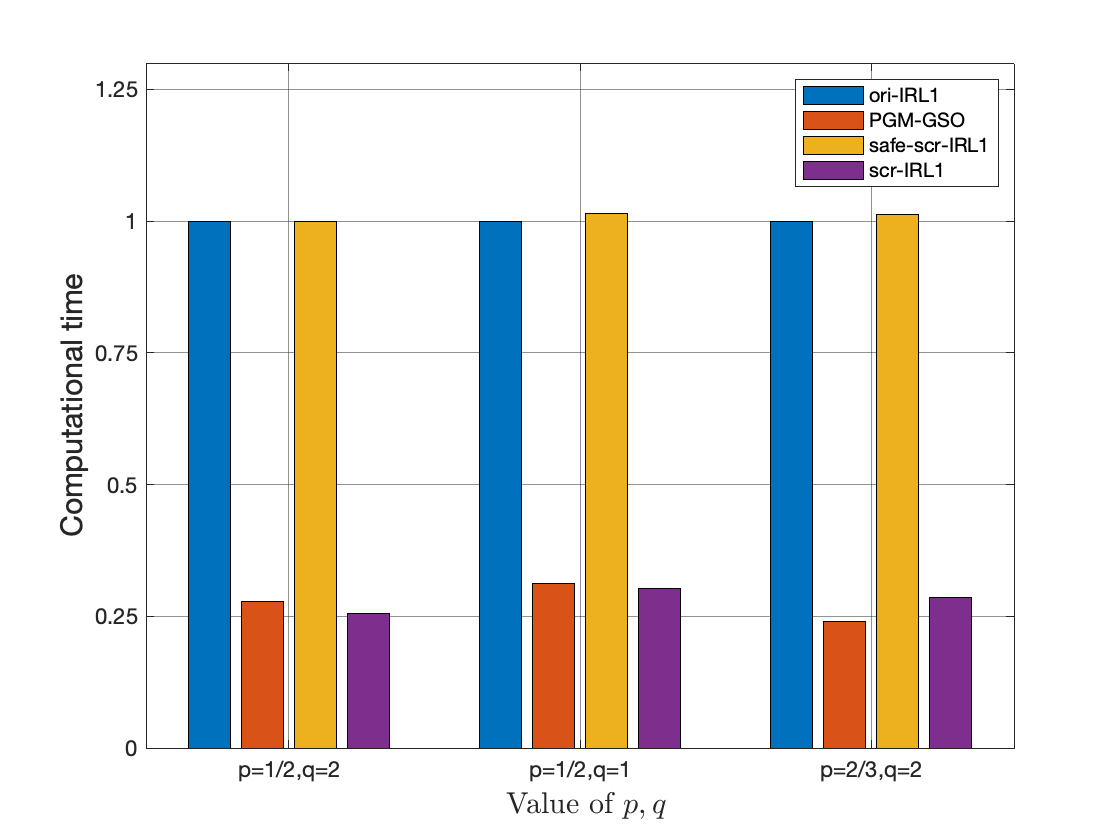} 
\end{minipage}
}
\subfigure[High-dimension data]{
\begin{minipage}[b]{0.48\textwidth}
\centering 
\includegraphics[width=1\textwidth]{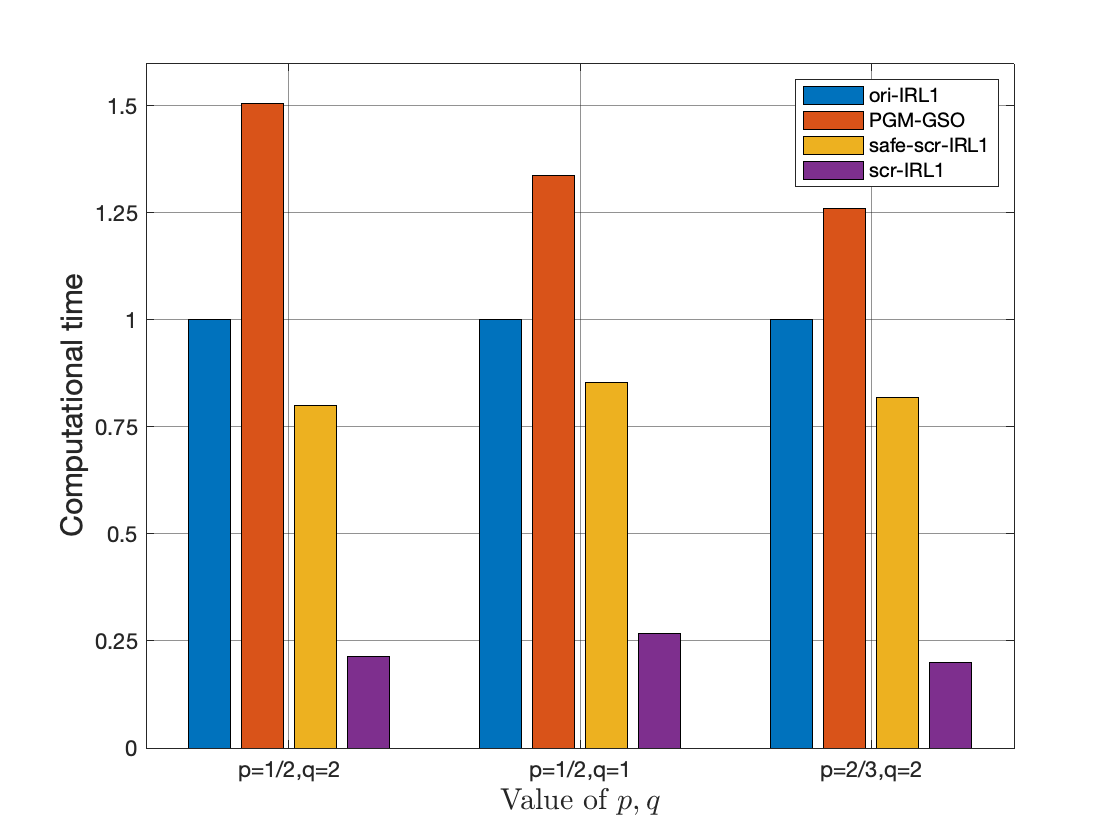} 
\end{minipage}
}
\caption{Illustration of computational time for different methods and screening rule strategies. Note that the computational time of ori-IRL1 is normalized to $1$.}
\label{time}
\end{figure}

\subsection{Experiments on Synthesized Data}
We first implement sparse signal recovery experiment to illustrate the effectiveness and efficiency of the proposed screening rule. Specifically, the benchmark algorithms considered in our comparison include the proximal gradient method for group sparse optimization (PGM-GSO) proposed in \cite{hu2017group},  the original IRL1 algorithm (ori-IRL1) and the safe heuristic screening rule method (safe-scr-IRL1) presented in \cite{ndiaye2017gap}. In particular, the IRL1 algorithm with our  screening rule strategy is abbreviated as scr-IRL1. Additionally, the termination criterion is same as that of \Cref{algo: overall}.

In this test, we set $Q = 20$ and fix the number of active groups $k=10$. Following \cite{hu2017group}, we generate $\bm{A}\in\mathbb{R}^{m\times n}$ satisfying $\bm{A}^{T}\bm{A} =\bm{I}$ with each of its entry following a standard Gaussian distribution. We let $\bm{x}_{\textrm{true}}$ denote the ground-truth vector to be estimated with group size $\vert\mathcal{G}_{i}\vert = 5$ for each $i \in [d]$, and non-zeros follow a standard Gaussian distribution. The predictor is generated obeying $\bm{y} = \bm{A}\bm{x}_{\textrm{true}} + \bm{\zeta}$, where $\zeta_{i}\sim \mathcal{N}(0,10^{-4})$. We consider two scenarios with $(m,n) = (500,2000)$ and  $(m,n) = (500,10000)$. Meanwhile, we compare different pairs of $(p,q)$ with $(p,q)=(\frac{1}{2},2)$, $(p,q)=(\frac{1}{2},1)$ and $(p,q)=(\frac{2}{3},2)$.

From~\Cref{time}, we first observe that the proposed scr-IRL1 generally outperforms all other methods with respect to the computational time under all scenarios. In particular, the computational time of ori-IRL1 can be reduced by a factor of at least $3$ times by equipping with the proposed screening rule, and the proposed screening rule is superior to the ones presented in \cite{ndiaye2017gap}. In addition, it also reveals that the number of wrongly discarded groups is significantly limited by applying our screening strategy.  On the other hand, we can see that our proposed screening rule becomes more efficient in a high-dimensional regime.

\par Next, we study the relationship between the computational gain and regularization parameter $\lambda$ and noise level $\sigma$, respectively. In particular, the computational gain is the ratio of the computational time of ori-IRL1 computational and that of scr-IRL1. We consider $(p,q) = (\frac{1}{2},2)$ and fix $k=10$ while varying the number of features $n$ by an increment $2000$. The results are shown in \cref{Gain}.

As observed in \Cref{timea}, larger $\lambda$ generally results in higher computational gain. From~\Cref{timeb}, we can see that larger noise level generally brings more significant computational gain. 
\begin{figure}[htbp]
\centering
\subfigure[]{
\begin{minipage}[b]{0.48\textwidth}
\centering 
\includegraphics[width=1\textwidth]{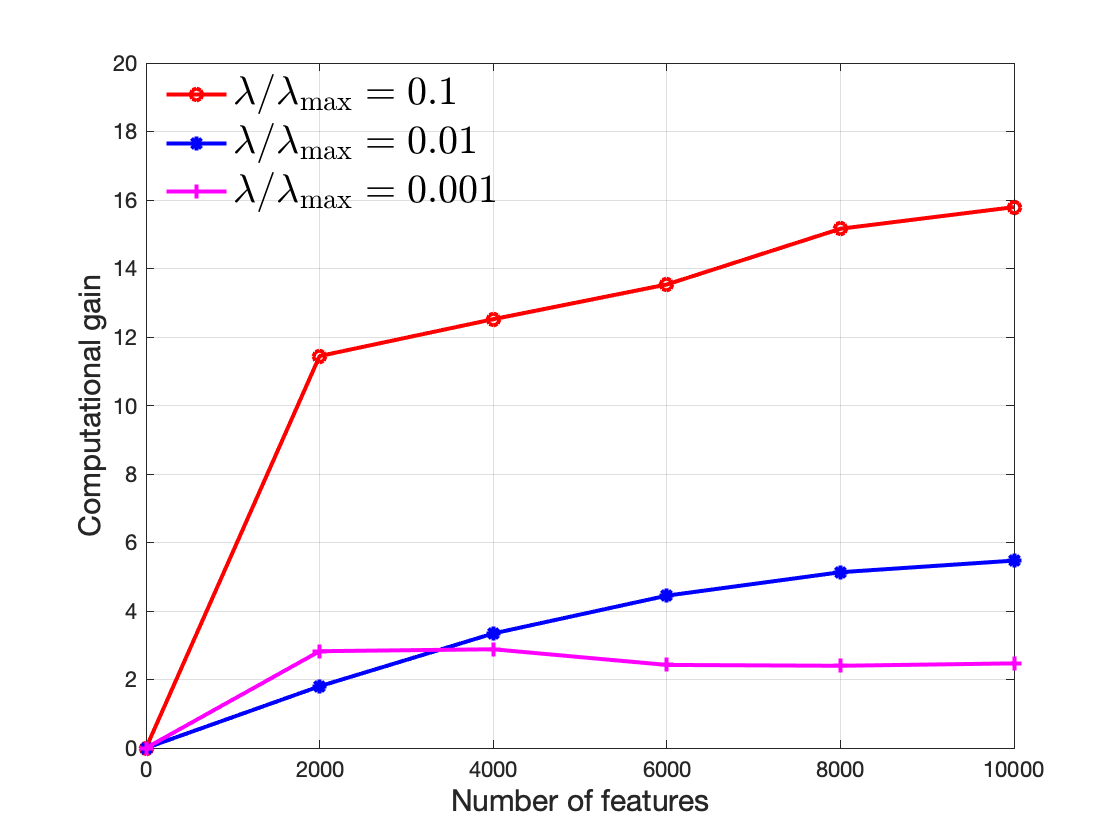}
\label{timea}
\end{minipage}
}
\subfigure[]{
\begin{minipage}[b]{0.48\textwidth}
\centering 
\includegraphics[width=1\textwidth]{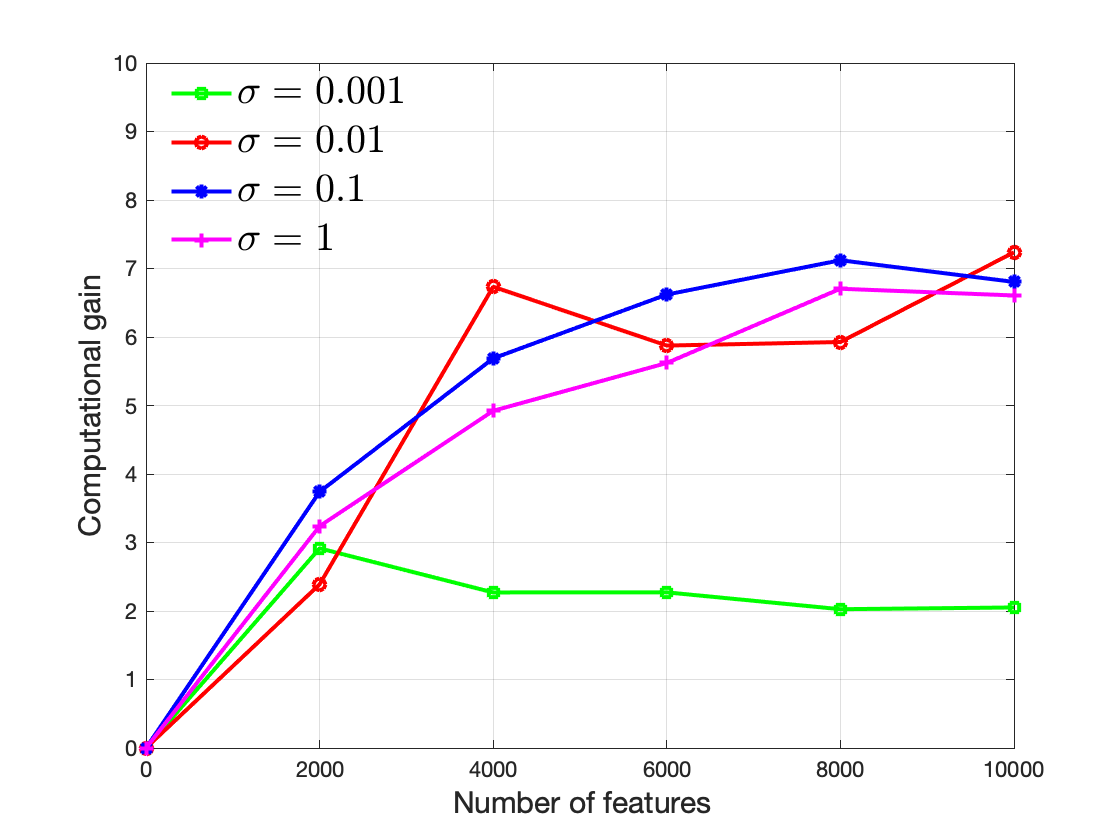}
\label{timeb}
\end{minipage}
}
\caption{Influence of $\lambda$ and $\sigma$ on computational gain. Left: varied $\lambda$ and fixed $\sigma = 0.01$. Right: varied $\sigma$ and fixed $\lambda = 0.01\lambda_{\max}$.}
\label{Gain}
\end{figure}

\subsection{Experiments on Real-world Data}
In this subsection, we carry out the experiments on real-world datasets to illustrate the efficiency of scr-IRL1. Specifically, we consider both the regression and classification tasks in machine learning, and the datasets are downloaded from LIBSVM dataset\footnote{\href{https://www.csie.ntu.edu.tw/~cjlin/libsvmtools/datasets/}{
bodyfat, pyrim and ionosphere }}, Kaggle\footnote{\href{https://www.kaggle.com/iabhishekofficial/mobile-price-classification}{mobileprice}} and UCI machine learning repository.\footnote{\href{https://archive.ics.uci.edu/ml/datasets/breast+cancer+wisconsin+(diagnostic)}{breastcancer}} The regression datasets include \textbf{bodyfat}, \textbf{pyrim} and \textbf{mobliprice}. The classification datasets include \textbf{ionosphere} and \textbf{breastcancer}. We follow the same strategy used in \cite{roth2008group} to form a group structure of the data used, which expands the data dimensions for features via a polynomial mapping. The expended dimensions for each dataset are $252\times455$, $74\times1755$, $2000\times950$, $351\times2805$ and $569\times2175$, respectively. Meanwhile, we set $Q=20$ and consider $(p,q) = (\frac{1}{2},2)$. 

\par From~\Cref{Fig: a}, we can see that the computational time is significantly reduced when the proposed screening rule is applied to IRL1.  Regarding the prediction accuracy shown in~\Cref{Fig: b}, the results of both algorithms are quite similar, which indicate that the proposed screening rule scarcely affects the solution quality.
\begin{table*}[htbp]
\caption{Results on the real-world datasets. For the regression task, the mean squared error is used to calculate the prediction error. For the classification task, the overall accuracy of the model is calculated as the ratio of the number of correct classifications and total classifications.}
\subtable[Computational Time(s)]{
\begin{minipage}[b]{0.495\textwidth}
\centering
\label{Fig: a}
\begin{tabular}{@{}cllcl@{}}
\toprule
Datasets                          & \multicolumn{4}{c}{Algorithms}                             \\ \cmidrule(l){2-5} 
\multicolumn{1}{l}{}              & \multicolumn{2}{l}{ori-IRL1} & \multicolumn{2}{l}{scr-IRL1}    \\ \midrule
\multicolumn{1}{c}{bodyfat}      & \multicolumn{2}{l}{$60.0203$}    & \multicolumn{2}{l}{$16.8846$}   \\
\multicolumn{1}{c}{pyrim}        & \multicolumn{2}{l}{$276.6057$}  & \multicolumn{2}{l}{$80.3436$}   \\
\multicolumn{1}{c}{mobileprice}   & \multicolumn{2}{l}{$842.4241$}  & \multicolumn{2}{l}{$408.8177$}   \\
\multicolumn{1}{c}{ionosphere}  & \multicolumn{2}{l}{$1.7127\times 10^3$}  & \multicolumn{2}{l}{$494.9049$}   \\
\multicolumn{1}{c}{breastcancer} & \multicolumn{2}{l}{$1.2308\times 10^3$}  & \multicolumn{2}{l}{$649.0551$} \\ \bottomrule
\end{tabular}
 \end{minipage}
 }
\subtable[Prediction Performance]{
\begin{minipage}[b]{0.495\textwidth}
\centering
\label{Fig: b}
\begin{tabular}{@{}cclcl@{}}
\toprule
Datasets                          & \multicolumn{4}{c}{Algorithms}                              \\ \cmidrule(l){2-5} 
\multicolumn{1}{l}{}              & \multicolumn{2}{l}{ori-IRL1} & \multicolumn{2}{l}{scr-IRL1}     \\ \midrule
\multicolumn{1}{c}{bodyfat}      & \multicolumn{2}{l}{$7.4108\times 10^{-4}$} & \multicolumn{2}{l}{$7.4198\times 10^{-4}$} \\
\multicolumn{1}{c}{pyrim}        & \multicolumn{2}{l}{$4.2491\times 10^{-4}$} & \multicolumn{2}{l}{$4.2490\times 10^{-4}$} \\
\multicolumn{1}{c}{mobileprice}   & \multicolumn{2}{l}{$0.1103$}   & \multicolumn{2}{l}{$0.1103$}   \\
\multicolumn{1}{c}{ionosphere}  & \multicolumn{2}{l}{$99.15\%$}   & \multicolumn{2}{l}{$99.15\%$}   \\
\multicolumn{1}{c}{breastcancer} & \multicolumn{2}{l}{$98.07\%$}   & \multicolumn{2}{l}{$98.07\%$}   \\ \bottomrule
\end{tabular}
\end{minipage}
}
\end{table*}
\par 
Next, we aim to confirm the efficiency of the proposed screening rule on the \textbf{breastcancer} dataset during the screening procedure and the KKT check procedure. We fix $(p,q) = (\frac{1}{2},2)$ while varying $\lambda$. 

In the first $20$ iterations, we record the number of screened groups, wrongly screened groups detected by the KKT check procedure and null groups of the solution obtained without screening rule. We report the ratios of these quantities in \cref{abla}. In particular, the ratio of the number of screened groups and the number of null groups of the solution obtained without screening rule uses shorthand RSN, and the ratio of the number of wrongly screened groups detected by the KKT check procedure and the number of null groups of the solution obtained without screening rule uses shorthand RWN.

From \cref{abla}, we can see that the proposed screening rule strategy can detect all of the null groups in a finite number of iteration while making mistakes with a extremely low probability.
\begin{figure}[htbp]
\centering
\subfigure[$\lambda = 0.001\lambda_{\max}$]{
\begin{minipage}[b]{0.31\textwidth}
\centering 
\includegraphics[width=1\textwidth]{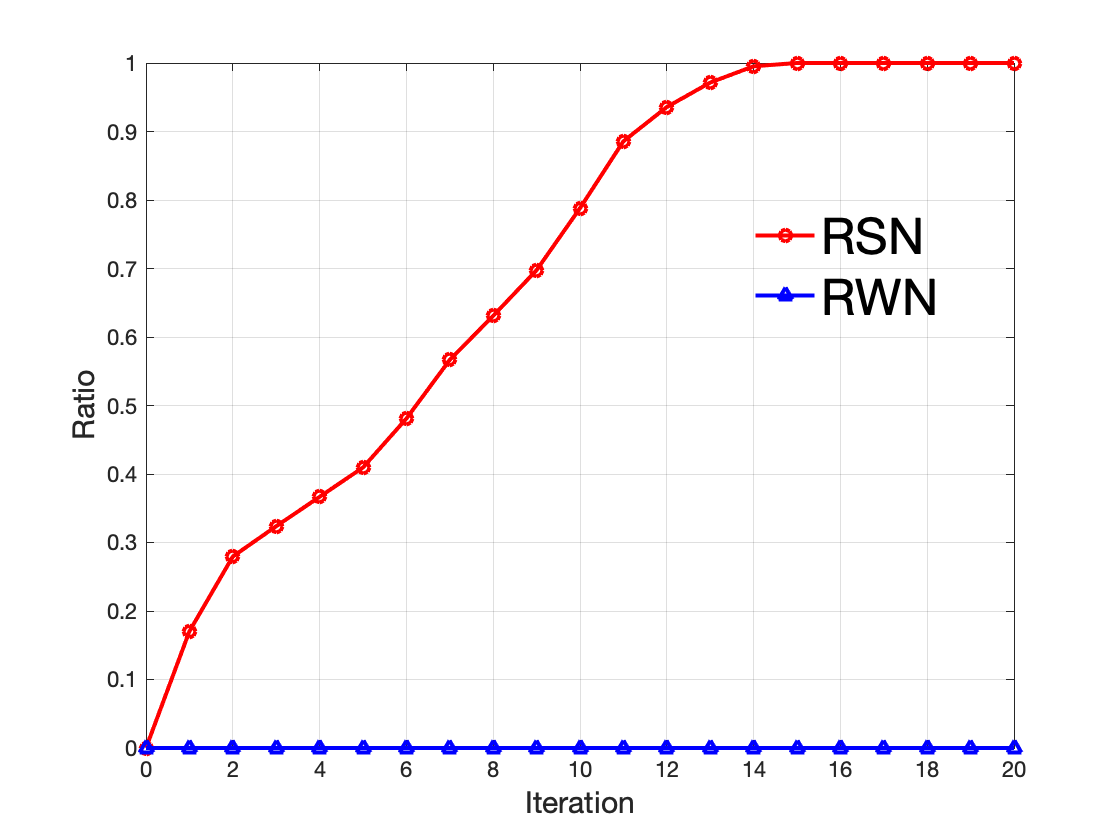}
\end{minipage}
}
\subfigure[$\lambda = 0.01\lambda_{\max}$]{
\begin{minipage}[b]{0.31\textwidth}
\centering 
\includegraphics[width=1\textwidth]{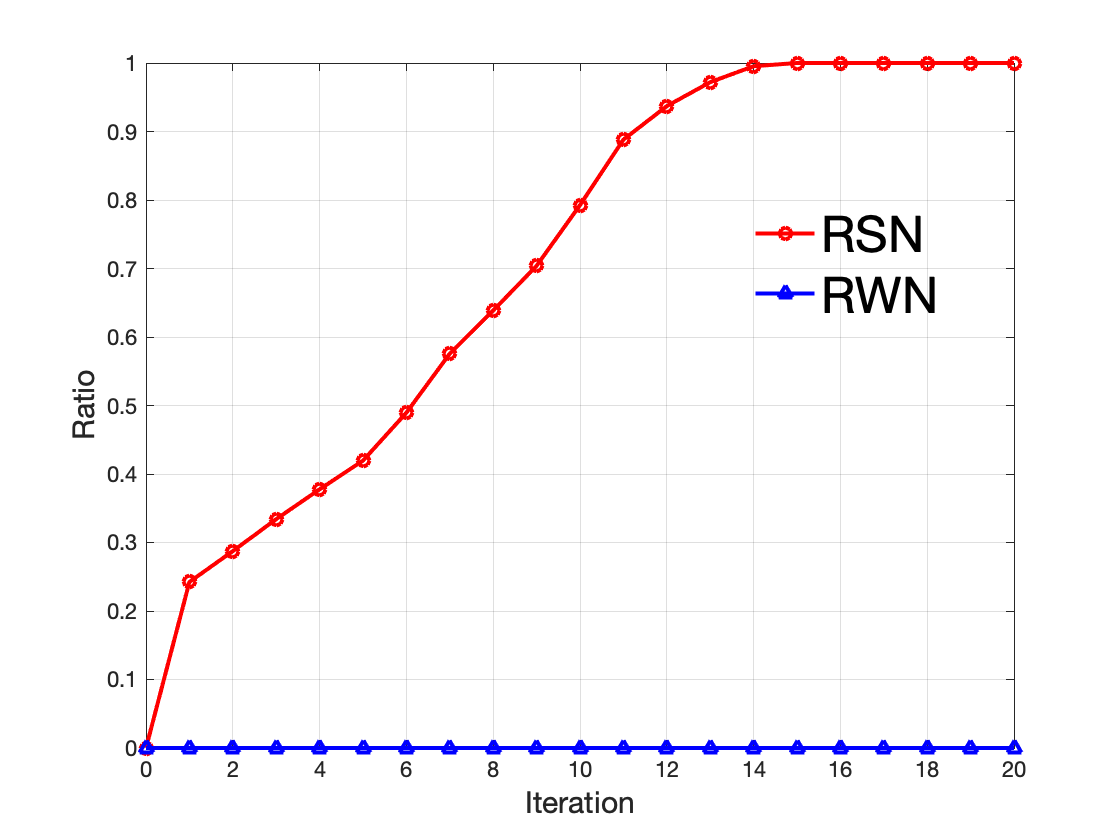}
\end{minipage}
}
\subfigure[$\lambda = 0.1\lambda_{\max}$]{
\begin{minipage}[b]{0.31\textwidth}
\centering 
\includegraphics[width=1\textwidth]{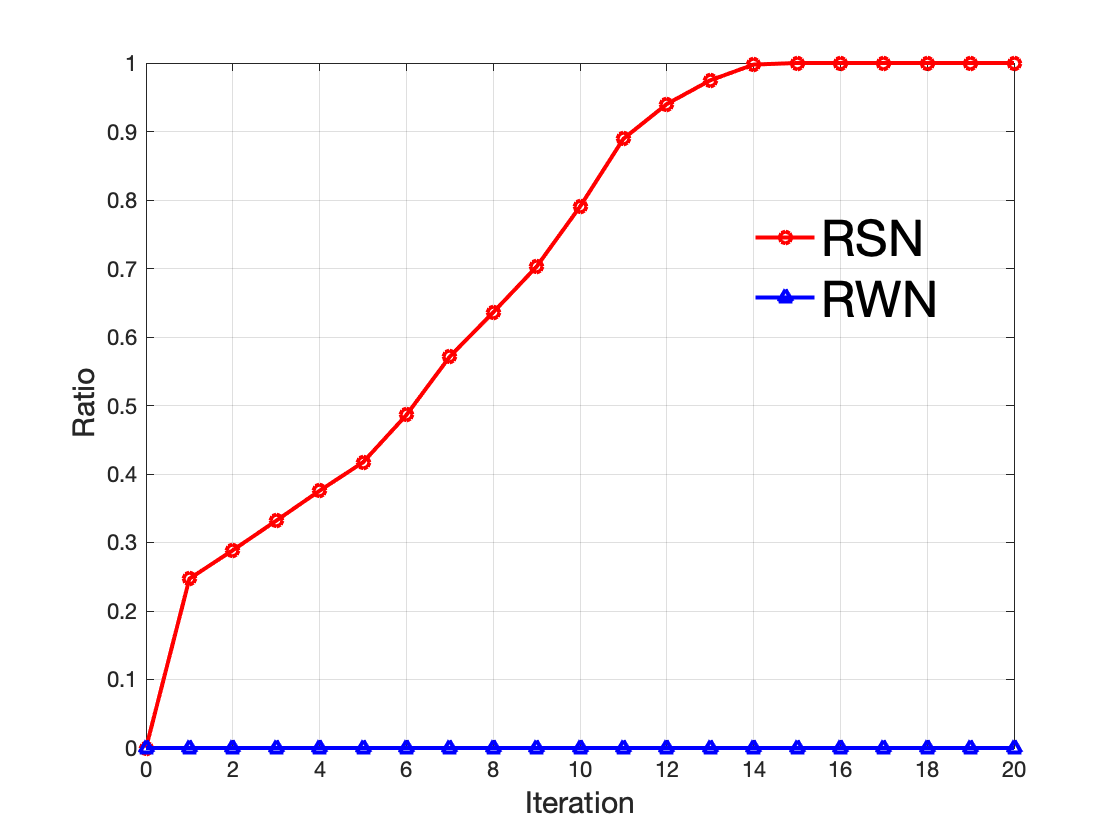}
\end{minipage}
}
\caption{Test of the efficiency and accuracy of the proposed screening rule on \textbf{breastcancer} dataset. }
\label{abla}
\end{figure}

\section{Conclusion}
In this paper, we proposed a screening rule strategy for structured optimization problem with nonconvex $\ell_{q,p}$ regularizer. The proposed screening rule could be applied before starting the subproblem solver, which was supposed to identify efficiently all null group features at the optimum. After solving the low-dimensional subproblem, we used a simple KKT check to guarantee finding the
optimal solution. We showed that those null group features could be identified and removed safely within a finite number of iterations. Numerical experiments demonstrated the empirical performance of the proposed screening rule on both synthesized and real-world datasets. 

\bibliography{neurips_2022}
\bibliographystyle{plain}

\end{document}